\documentclass[twoside]{article}

\usepackage{amsmath}
\usepackage{amssymb}
\usepackage{amsthm}
\usepackage{algorithm}
\usepackage{algorithmic}
\usepackage{array}
\usepackage{ifpdf}
\usepackage{cite}
\usepackage{setspace}
\usepackage{graphicx}
\usepackage{subfigure}
\usepackage{color}

\newtheorem{proposition}{Proposition}

\usepackage{hyperref}

\usepackage[accepted]{aistats2019}

\begin{document}

\twocolumn[

\aistatstitle{Adaptive Clinical Trials: Exploiting Sequential Patient Recruitment and Allocation}
\aistatsauthor{ Onur Atan \And William R. Zame \And  Mihaela van der Schaar }
\aistatsaddress{ UCLA \And  UCLA \And UCLA} 
]

\begin{abstract} \vspace{-.2in}
Randomized Controlled Trials (RCTs) are the gold standard for comparing the effectiveness of a new treatment to the current one (the control).  Most RCTs allocate the patients to the treatment group and the control group by uniform randomization.  We show that this procedure can be highly sub-optimal (in terms of learning) if -- as is often the case -- patients can be recruited in cohorts (rather than all at once), the effects on each cohort can be observed before recruiting the next cohort, and the effects are heterogeneous across identifiable subgroups of patients.  We formulate the patient allocation problem as a finite stage Markov Decision Process in which the objective is to minimize a given weighted combination of type-I and type-II errors.  Because finding the exact solution to this Markov Decision Process is computationally intractable, we propose an algorithm --  \textit{Knowledge Gradient for Randomized Controlled Trials} (RCT-KG) -- that yields an approximate solution. We illustrate our algorithm on a synthetic dataset with Bernoulli outcomes and compare it with uniform randomization.  For a given size of trial our method achieves significant reduction in error, and to achieve a prescribed level of confidence (in identifying whether the treatment is superior to the control), our method requires many fewer patients.  Our approach uses what has been learned from the effects on previous cohorts to recruit patients to subgroups and allocate patients (to treatment/control) within subgroups in a way that promotes more efficient learning.  \end{abstract}

\section{Introduction} \vspace{-.1in}

Randomized Controlled Trials (RCTs) are the gold standard for evaluating new treatments.  Phase I trials are used to evaluate safety and dosage, Phase II trials are used to provide some evidence of efficacy, and Phase III trials (which are the subject of this paper) are used to evaluate the effectiveness of the new treatment in comparison to the current one.  A typical question that a Phase III RCT is intended to answer is ``\textit{In population W is drug A at daily dose X more efficacious in improving Z by Q amount over a period of time T than drug B at daily dose Y?}'' \cite{friedman1998fundamentals}. 

RCTs are useful because they create a treatment group and a control group that are as similar as possible except for the treatment used.  Most RCTs recruit patients from a prescribed target population and uniformly randomly assign the patients to treatment groups using repeated uniform randomization (perhaps adjusted to deal with chance imbalances). This approach is optimal if all patients are recruited at once or if the outcomes for previous patients cannot be observed when recruiting new patients or if the patient population is (or is thought to be) homogeneous. However, in many circumstances, patients are (or can be) recruited in cohorts, the outcomes for patients in previous cohorts can be observed when recruiting a new cohort, and the population contains identifiable subgroups for which differences in effects might be expected.  (For example, different effects of treatment might be expected for patients with different genetic mutations \cite{moss2015efficacy}; see \cite{hebert1999multicenter} and Table \ref{table:rct_example} for additional examples.)  In such situations, the information learned from previous cohorts can be used in recruiting and allocating patients in the new cohort: the optimal policy should not necessarily recruit the same number of patients to each identifiable subgroup or allocate equal numbers of patients within a subgroup to the treatment  and the control.  We illustrate this point in the Experiments.

Our goal in this paper is to develop a procedure that prescribes two things: (i) the number of patients in each cohort to recruit from each subgroup, (ii) the allocation of these patients to treatment or control in order to minimize the error (type-I or type-II or a given convex combination) in identifying the patient subgroups for which the treatment is more/less effective than the control. Our work differs from recent work on Bayesian clinical trials~\cite{berry2006bayesian} in which the the information obtained from the previous cohorts are used only for treatment allocation; in our work, the information from previous cohorts is used both for patient recruitment and for allocation to treatment or control. As an example, consider the RCT setting in \cite{barker2009spy} for neoadjuvant chemotherapy in which  two subgroups are identified based on the hormone receptor status, human epidermal growth factor receptor 2 (HER2) status, and MammaPrint11,12 status. If our procedure were to be used in this actual trial to recruit 100 patients in each cohort, the initial cohort would consist of 50 patients from each subgroup and would allocate them uniformly to treatment/control.  However, in the second and succeeding cohorts, we would use the observed outcomes from earlier cohorts to recruit more patients from the subgroup in which uncertainty about treatment efficacy was larger and, within each subgroup, we would allocate more patients to whichever of treatment/control had displayed larger variance. 

Our first contribution is to formalize the learning problem as a finite stage Markov Decision Problem (MDP) in which the designer recruits $N$ patients over $K$ steps and observes the outcomes of step $k-1$ before taking step $k$.  However, because the action and state spaces of this MDP are very large, solving this MDP by dynamic programming is computationally intractable.  We therefore propose a computationally tractable greedy algorithm \textit{Knowledge Gradient for Randomized Controlled Trials} (RCT-KG)  that yields an approximate solution. We illustrate the effectiveness of our RCT-KG algorithm in a set of experiments using synthetic data in which outcomes are drawn from a Bernoulli distribution with unknown probabilities.  In particular, we show that, keeping the sizes of the trial and of the cohorts fixed, RCT-KG yields significantly smaller expected error; conversely, in order to achieve a given level of confidence in identifying whether the treatment is better than the control, RCT-KG requires many fewer patients/cohorts. 

Our approach makes a number of assumptions.  The first is that patients can be recruited in cohorts, and not all at once.  The second is that the outcomes for patients in each cohort are realized  and can be observed before recruiting and allocating patients for the succeeding cohort.  The third is that subgroups are identified in advance.  The fourth is that, for each cohort (after the first), the number of patients recruited in each subgroup and the allocation of patients (to treatment or to control) within each subgroup can be chosen to depend on the observations made from previous cohorts. These assumptions are strong and certainly are not satisfied for all RCTs, but they are satisfied for {\em some} RCTs, and for those our approach offers very significant improvements over previous approaches to the speed and accuracy of learning.

It is important to understand that these assumptions do {\em not} imply that our approach is entirely unblinded.  In each cohort, the assignment of patients within each subgroup to treatment/control is completely blind to both patients and clinicians; it is only {\em after} the assignment is made and the outcome is realized that the clinicians learn which patients received the treatment and which received the control.  Moreover, In each cohort, the assignment of patients within each subgroup to treatment/control is completely random, although the particular randomization used for a particular subgroup in a particular cohort depends on the outcomes that have been observed for previous cohorts.
 
\begin{table*}
\scriptsize
\centering
\begin{tabular}{ | c | c | p{3cm} | p{3cm} | p{6cm} | }
    \hline
    Study & Size & Treatment & Primary Outcome & Result   \\ \hline
    \cite{hacke1995intravenous} & $620$ & rt-PA (alteplase) & Barthel Index (BI) at $90$ days & Treatment is effective in improving the outcome in a defined subgroup of stroke patients. \\ \hline
    \cite{hebert1999multicenter} & $838$ & Red cell transfusion & $30$ days mortality & effective among the patients with Apache 2 score less than equal to $20$ and age less than $55$. \\ \hline
    \cite{hacke2008thrombolysis} & $821$ & Intravenous thrombolysis with alteplase & disability at $90$ days & As compared with placebo, intravenous alteplase improved clinical outcomes significantly.  \\ \hline
    \cite{moss2015efficacy} & $69$ & Ivacaftor &  ppFEV1 in week 24 & Ivacaftor significantly improves lung function in adult patients with R117H-CFTR. \\ \hline
    \end{tabular}
\caption{RCT Examples in the Literature}
\label{table:rct_example}
\end{table*}

\section{Related Work}\vspace{-.1in}

The most commonly used procedure to allocate patients into treatment and control groups is ``repeated fair coin-tossing'' (uniform randomization).  One potential drawback to this approach is the possibility of unbalanced group sizes when the set of patients (or the set of patients in an identified subgroup) is small \cite{friedman1998fundamentals}. Hence, investigators often follow a {restricted} randomization procedure for small RCTs, such as {\em blocked randomization} \cite{lachin1988randomization} or {\em adaptive bias-coin randomization} \cite{schulz2002allocation}. These procedures have the effect of  assigning more patients to  treatment/control groups to prevent an imbalance between the groups. An alternative, but less frequently used procedure is \textit{covariate-adaptive randomization} in which patients are assigned to treatment groups to minimize \textit{covariate imbalance} \cite{moher2012consort}. These approaches are efficient when patients are recruited at one time or the outcomes of previous cohorts are not observed . However, as we have noted and can be seen in Table \ref{table:rct_example}, it is often the case that patients are recruited sequentially in cohorts and the outcomes of previous cohorts can be observed before recruiting the next cohort.  There is a substantial literature concerning such settings \cite{lewis1990sequential,whitehead1997design}, but it focuses on the decision of whether to terminate the trial, rather than how to allocate the next cohort of patients to the treatment or control groups, which is the focus of our paper.  

 \textit{Response adaptive randomization}  uses information about previous cohorts to allocate patients in succeeding cohorts: the probability of being assigned to a treatment group is increased if responses of prior patients in that particular group has favorable~\cite{hu2006theory,berry2006bayesian}. However, the goal of the most of these approaches is to improve the benefit to the patients in the trial rather than to learn more about the comparison between the treatment and the control. Multi-Armed Bandits (MABs) constitute a general mathematical decision framework for resource (patient) allocation with the objective of  maximizing the cumulative outcomes (patient benefits) \cite{auer2002finite,gittins2011multi,agrawal2012analysis}. However, \cite{villar2015multi} shows that although MABs achieve greater patient benefit, they suffer from poor learning performance because they allocate most of the patients to favorable actions (treatments). Hence, \cite{villar2015multi} proposes a variation on MAB algorithms that allocates patients to control action at times to maintain patient benefit while mitigating the poor learning performance. Some of the existing work on Bayesian RCTs~\cite{berry2006bayesian} does use the information from previous cohorts to allocate patients in order to improve the learning performance. The work closest to ours may be \ref{bhatt2016} which specifically addresses the special case of a trial with two identified subgroups.  However the only adaptation considered is to entirely cease recruiting patients into one of the subgroups, not to change the number of patients recruited into each subgroup.  Our paper prescribes a principled algorithm to improve the learning performance that adjusts both the number of patients recruited into each subgroup and the proportion of patients allocated to treatment/control within each subgroup.

Another line of literature relevant to ours is family of policies known as Optimal Computing Budget Allocation (OCBA) ~\cite{chen1995effective,chen2003optimal}. These policies are derived as an optimization problem to maximize the probability of later identifying the best alternatives by choosing the measurements. The approach in these policies is by approximating the objective function with lower and upper bounds. In this paper, we provide a Bayesian approach to RCT design and model the problem as a finite stage MDP.  The optimal patient allocation policy can be obtained by solving the Dynamic Programming (DP); for the case of discounting, the Gittins Index ~\cite{gittins2011multi,whittle1980multi} provides an (optimal) closed form solution to the MDP.  However, both of these approaches are computationally intractable for problems as large as typical RCTs.

Knowledge Gradient (KG) policies provide an approximate, greedy approach to the MDP. However, in our setting, the action space for the MDP, which contains all possible allocation choices, is very large and hence  KG policies are again not  computationally tractable. (Morever, KG policies typically assume multivariate normal priors for the measurements~\cite{frazier2008knowledge,frazier2009knowledge}.). \cite{chen2013optimistic, chen2015statistical} proposes a variation of  KG policies that they deemed Opt-KG and that selects the action that may generate maximum possible reduction in the errors.  In our setting, their approach would require that patients be recruited one at a time and that the outcome of the action (treatment or control) chosen for each patient from each  subgroup be observable before the next patient is recruited.  Those requirements are not appropriate for our setting, in which we  observe only a (noisy) signal about the true outcome of the selected treatment and we recruit patients in cohorts. not one at a time.  (It would typically be completely impractical to allocate patients one at a time: for  any realistic time to observation and total number of patients, doing so would result in a clinical trial that would last for many years.)  Moreover, our approach allow for a much broader class of outcome distributions (exponential families, which includes Bernoulli distributions and many others) and we allocate all patients in a cohort at each stage.  Our  RCT-KG algorithm might be viewed as generalizing  Opt-KG in all of the aspects mentioned.

\section{A Statistical Model for a Randomized Clinical Trial} \vspace{-.1in}

Our statistical model for a RCT has four components: the patient population, the given patient subgroups, the treatments and the  treatment outcomes. Write $\mathcal{W}$ for the patient population and  $X$ for a prescribed partition of $\mathcal{W}$ into patient subgroups. Write $Y = \{0,1\}$ for the action space; $0$ represents the  control action and 1 represents the treatment action.  Let $Z$ be the outcome space; without much loss we assume $Z \subset \mathbb{R}$ with minimum $z_{\min}$ and maximum $z_{\max}$.  ($Z$ might be continuous or discrete.)
     
We wish to allocate a total of $N$ patients in $K$ steps/cohorts over a total time $T$ in order to identify the more efficacious treatment for each subgroup.  As discussed in the Introduction, we make the following assumptions:
\begin{enumerate} \vspace{-.1in}
\item Patients can be be recruited in cohorts. \vspace{-.1in}
\item The outcomes for patients in each cohort are realized and can be observed before recruiting and allocating patients for the succeeding cohort. \vspace{-.1in}
\item Subgroups can be identified in advance. \vspace{-.1in}
\item For each cohort after the first, the number of patients recruited in each subgroup and the allocation within each subgroup can be chosen as a function of the observations from previous cohorts.    \vspace{-.1in}
\end{enumerate}
In addition, we make one technical assumption:
\begin{enumerate} \vspace{-.1in}
\setcounter{enumi}{4}
   \item The outcome distribution belongs to the exponential family. \vspace{-.1in}
   \end{enumerate}

In the next subsection, we give a brief description of the exponential family of distributions and Jeffrey's prior on their parameters.

\subsection{Exponential Families and Jeffrey's Prior} \vspace{-.1in}

Let $\Theta$ be a  parameter space. Fix functions $G: Z \rightarrow \mathbb{R}^d$ and $h: Z \rightarrow \mathbb{R}$. The $d$-dimensional parameter exponential family with sufficient statistic 
$G$ and parametrization $\theta$, relative to $h$,  is the family of distributions defined by the following densities:
$$
p(z | \theta) =  \Phi(\theta ) h(z) \exp\left(\theta {\cdot} {G}(z)\right)
$$
where $\Phi(\theta)$ is uniquely determined by the requirement that $p(\cdot | {\theta})$ is a probability density; hence 
\begin{eqnarray*}
\int_{-\infty}^{\infty} p(z | {\theta}) d\,z &=& 1\\
\Phi({\theta}) &=& \left( \int_{-\infty}^{\infty} h(z) \exp({\theta} {\cdot} {G}(z)) d\,(z) \right)^{-1}
\end{eqnarray*}
 An alternative expression is: 
$$
p(z | {\theta}) =  h(z) \exp\left({\theta} {\cdot} {G}(z) - F({\theta})\right)
$$
where $F({\theta}) = - \log \Phi({\theta})$. Write $\mu$ for  the expectation: 
$\mu({\theta}) = \mathbb{E}_{Z | {\theta}}\left[ Z\right]$. Different choices of the various ingredients lead to a wide variety of different probability densities; for example choosing $G(z) = z, h(z) = 1, \theta = \ln \frac{q}{1-q}$ generates a Bernoulli distribution. Other choices lead to Poisson, exponential and Gaussian distributions. 

In this paper, we want a Bayesian ``non-informative'' prior that is invariant under re-parametrization of the parameter space. To be specific, we  use  Jeffrey's prior, which is proportional to the square root of the Fisher information $I({\theta})$. In the case of the exponential family, the Fisher information is the second derivative of the normalization function, i.e., $I({\theta}) = \sqrt{ \big| F''({\theta}) \big| }$. Under  Jeffrey's prior, the posterior on the parameter 
$\theta$ after $n$ observations is given by
$$
p({\theta} | z_1, \ldots, z_n) \propto \sqrt{ \big| F''({\theta}) \big| } \exp\left( \sum_{i=1}^n {\theta} {\cdot} {G}(z_i) - n F({\theta})\right)
$$
Given $n$ outcomes $(z_1, z_2, \ldots, z_n)$, we can summarize the information needed for the posterior distribution by ${s} = \left[ s_0, s_1\right] = \left[ \sum_{i=1}^n {G}(z_i), n\right]$.  The posterior is then proportional to $ \sqrt{ \big| F''({\theta}) \big| }\exp\left({\theta} {\cdot} {s}_0 - s_1 F(\theta)\right)$. 

Given a subgroup $x \in X$ and a treatment $y \in Y$ Let ${\theta}_{x,y}$ be the true parameter for subgroup $x$ and treatment $y$. (Of course ${\theta}_{x,y}$ is not known.) The result of treatment $y$ on a patient in subgroup $x$ is referred to as the {\em treatment outcome} and is assumed to be drawn according to the exponential family distribution, i.e., $Z \sim p(\cdot | {\theta}_{x,y})$. Write $\mu({\theta}_{x,y})$ for the true expected outcome of the treatment $y$ on subgroup $x$ and define the {\em treatment effect} for the parameters $\theta_0, \theta_1$ as $E({\theta}_0, {\theta}_1) = \mu({\theta}_1) - \mu({\theta}_0)$;  the treatment effect on subgroup $x$ is  $E({\theta}_{0,x}, {\theta}_{1,x})$. 

\subsection{Treatment Effectiveness } \vspace{-.1in}

Given a threshold $\tau \geq 0$ (set be the designer), we define 
\[ \nu(x) =
  \begin{cases}
    1       & \quad \text{if } \frac{E({\theta}_{x,0}, {\theta}_{x,1})}{\mu({\theta}_{x,0})} \geq \tau \\
    0  	    & \quad \text{if  otherwise}
  \end{cases}
\] 
so $\nu(x) = 1$ if the treatment is sufficiently better than the control, in which case we say the treatment is {\em effective}. (For example, see  \cite{farrar2000defining}, in which the goal is to identify whether reduction in pain by the treatment with respect to control is more than $33\%$.) We define the {\em positive set} ${H}^{+} = \{ x \in {X} : \nu(x) = 1 \}$ to be  the set of subgroups for which the treatment is effective and the {\em negative set} ${{H}}^{-} = {X} \setminus {H}^{-}$ to be the complementary set of subgroups for which the treatment  is ineffective. 

Given the dataset, any algorithm can only produce a set of subgroups in which treatment is {\em estimated} to be effective; write ${H}^+_{\rm est}$ for this set of subgroups and ${{H}}^-_{\rm est}$ for the complementary set of subgroups.  A type-I error occurs if a subgroup $x \in {H}^{+}$ is  in 
${{H}}^-_{\rm est}$ (i.e. treatment is actually effective but is estimated to be ineffective); a type-II error occurs if a subgroup $x \in {{H}}^{-} $ is in ${H}^{+}$ (i.e. treatment is actually ineffective but is estimated to be effective).  For a given estimated set ${H}^+_{\rm est}$, the magnitudes of type-I and type-II errors are
\begin{align*}
\mathrm{e}^K_1 &= \sum_{x \in {X}} 1\left(x \in {H}^{+} \right) 1\left(x \in {H}^-_{\rm est} \right)& \notag \\ 
\mathrm{e}^K_2 &= \sum_{x \in {X}} 1\left(x \in {H}^{-} \right) 1\left(x \in {H}^+_{\rm est}  \right)&
\end{align*} 
Given $\lambda \in [0,1]$ the {\em total error} is: 
$$\mathrm{e}^K = \lambda \mathrm{e}^K_1 + (1- \lambda)\mathrm{e}^K_2$$ 
where $\lambda$ is a parameter that is selected by the designer based on the designer's view of the importance of type-I and type-II errors. (We use the superscript $K$ to indicate that we are computing errors after $K$ cohorts have been recruited.)

\section{Design of a RCT as a Markov Decision Problem} \vspace{-.1in}

In this subsection, we model the RCT design problem as a finite step non-discounted MDP. A finite step MDP consists of number of steps, a state space, an action space, transition dynamics and a reward function. We need to define all the components of the MDP.   

We are given a budget of $N$ patients to be recruited in $K$ steps. At time step $k$, the designer decides to recruit $M_k$ patients; of these $u_k(x,y)$ are from subgroup $x$ and are assigned to treatment $y$, so  
$
\sum_{x \in X} \sum_{y \in Y} u_k(x,y) = M_k.
$  
Having made a decision $U_k = \{ u_k(x,y) \}$ in step $k$, the designer observes the outcomes $W_k = \{ W_k(x,y) = \sum_{j=1}^{u_k(x,y)} G(Z_j): Z_j \sim \mathbb{P}(\cdot | \theta_{x,y}) \}$. 

Write $\bar{M}_{k-1} = \sum_{\ell=0}^{k-1} M_\ell$ for the number of patients recruited through step $k-1$. We define a filtration $\left( \mathcal{F}_k \right)_{k=0}^K$ by setting $\mathcal{F}^k$ to be the sigma-algebra generated by the  decisions and observations through step $k-1$: $\{ U^0, W^0, U^1, W^1 \ldots, U^{k-1}, W^{k-1} \}$. We write $\mathbb{E}_k\left[ \cdot\right] = \mathbb{E}\left[ \cdot | \mathcal{F}^k \right]$ and $\operatorname{Var}_k\left[ \cdot \right] = \operatorname{Var} \left[\cdot | \mathcal{F}^k \right]$. Recruitment and allocation decisions are restricted to be $\mathcal{F}^k$-measurable so that decisions to be made at each step depends only on information available from previous steps.

The state space for the MDP is the space of all possible distributions under consideration for $\{\theta_{x,y}\}$. Let $S^k$ denote the $2 X \times (d+1) $ \textit{state matrix}  that contains the hyper-parameters of posterior distribution of the outcomes for both treatment and control actions for all $(x,y) \in X \times Y$ in the $k$th step. Define 
$\mathcal{S}^k$ to be the all possible states at the $k$th step, that is, 
$
\mathcal{S}^k = \bigg\{ S^k = \left[s_{x,y}^k \right] :  s_{x,y}^k = \left[ s_{x,y,0}^k, s_{x,y,1} \right] \bigg\}
$
where $s_{x,y,0}^k$ is the $d$-dimensional cumulative sufficient statistic and $s_{x,y,1}$ is the number of samples from subgroup $x$ with treatment action $y$. 

The action space at step  $k$ is the set of all possible pairs of $(M_k, U_k)$ with $M_k \leq N - \bar{M}_{k-1}$ and $\sum_{x,y} u_k(x,y) = M_k$. Taking an action $a_k = (M_k, U_k)$ means recruiting $M_k$ patients in total, of whom $u_k(x,y)$ will be  from subgroup $x$ and assigned to treatment $y$. Fix the decision stage as $k$. When the designer selects one of the actions, we use  Bayes rule to update the distribution of $\theta_{x,y}$ conditioned on ${F}^k$ based on outcome observations of $W_k$, obtaining a posterior distribution conditioned on $\mathcal{F}^{k+1}$. Thus, our posterior distribution for $\theta_{x,y}$ is proportional to $\sqrt{\big| F''({\theta}) \big |} \exp\left(\theta {\cdot} s_{x,y,0}^k - F({\theta}) s_{x,y,1}^k \right)$. The parameters of the posterior distribution can be written as a function of $s^k$ and $W_k$. Define $S^{k+1} = T(s^k, a_k, W_k)$ to be the transition function given observed treatment outcome $W_k$ and  $\mathbb{P}(S^{k+1} | S^k, a_k)$ to be the posterior state transition probabilities conditioned on the information available at step $k$. Having taken an allocation action $a_k = (M_k, U_k)$, the state transition probabilities are: 
$S^{k+1} = s^k + \left[ W^k, U^k \right]$ with posterior predictive probability $\mathbb{P}(W| s^k)$. 

For a state $s = (s_0, s_1)$, the action space is the set of pairs of $(m, u)$ with $m$ less than or equal to the remaining patient budget, and elements in $u$ summing up to $m$. Denote this set as $A(s)$.

Given the state vector $s$, write $P_x(s)$ for the posterior probability that the treatment is effective, conditional on $\theta_{x,0}$ and $\theta_{x,1}$ being drawn according to the posterior distributions. Formally, 
\begin{align}
P_x(s) = \mathbb{P}\left(\frac{ E(\theta_0, \theta_1)}{\mu(\theta_0)} \geq \tau \bigg| \substack{\theta_0 \sim \mathcal{P}_{\theta | s_{x,0,0}, s_{x,0,1}} \\  \theta_1 \sim \mathcal{P}_{\theta | s_{x,1,0}, s_{x,1,1}}} \right). \notag
\end{align}
We can now compute estimated positive and negative sets, $H^{+}_{est}, H^{-}_{est}$. If $x \in H^{-}_{est}$ (i.e. the treatment is estimated to be ineffective for the subgroup $x$) then the probability that $x \in H^+$ (i.e. the treatment is actually effective) is $P_x(s)$. Similarly, if $x\in H^{+}_{est}$ then the probability that $x \in H^{-}$ is $1 - P_x(s)$. Hence, given $H^{+}_{est}, H^{-}_{est}$  the posterior expected total error is:
\begin{align}
\sum_{x \in X} \lambda &P_x(s) 1(x \in H^{-}_{est})& \notag \\ 
&+  (1-\lambda) (1-P_x(s)) 1(x \in H^{+}_{est})& \label{eqn:eqn:opt_identify}
\end{align}

The following proposition identifies the set that minimizes this posterior expected total error.  
\begin{proposition}
Given the terminal state $s$, the set that minimizes this posterior expected total error is  $H^{+}_{est} = \{ x \in X: P_x(s) \geq 1- \lambda \}$.
\end{proposition}
\begin{proof}
Fix a particular subgroup $x$. If the subgroup $x$ is in the estimated positive set, then total expected error for subgroup $x$ is given by $\lambda P_x(s)$. Similarly, if the subgroup $x$ is in the estimated negative set, then total expected error for subgroup $x$ is given by $(1-\lambda) (1-P_x(s))$. When $P_x(s) \geq 1 - \lambda$, it holds that $\lambda P_x(s) \geq (1-\lambda) (1-P_x(s))$. In that case, $x \in H^{+}_{est}$ minimizes the total expected posterior error.  
\end{proof}
Now define  $g$ by 
\begin{align}
g(x; \lambda) = \lambda (1 - x) 1 (x \geq 1 -\lambda) + (1 - \lambda) x 1 (x < 1-\lambda). \notag
\end{align}
Then, the posterior expected total error can be written as $\mathrm{e}^K = \sum_{x \in X} g\left(P_x(s^K); \lambda\right)$. We'll omit using $\lambda$ in the rest of the paper for notational brevity. We define the reward function $R : S \times A \rightarrow \mathbb{R}$ as the decrease in the  posterior expected total error that results from taking a particular action in state $s$; i.e.
\begin{align}
R(s, a) =\sum_{x \in X} \mathbb{E}\left[ g\left(P_x(s^k)\right) - g\left(P_x(s^{k+1})\right) \big| s^k = s, a^k = a \right] \notag 
\end{align}
where the expectation is taken with respect to the outcome distributions of the treatment and control actions given the state vector $s$. A {\em policy} is a mapping $\pi: S \rightarrow A$ from the state space to the action space, prescribing the number of patients to recruit from each subgroup and how to assign them to treatment groups. The value function of a policy $\pi$ beginning from state $s^0$ is 
\begin{align}
V^{\pi}(s^0) &=  B(s^0) - \sum_{x \in X} \mathbb{E}^{\pi} \left[ g\left(P_x(s^k); \lambda\right) \right]& \notag \\ 
&= \sum_{k=0}^K \mathbb{E}^{\pi} \left[ R(S^k, A_k) \right] &
\end{align}
where $B(s^0) = \sum_{x \in X} g\left(P_x(s^0); \lambda\right) $ and the expectation is taken with respect to the policy $\pi$. Our learning problem is the $K$-stage MDP with tuple: 
$
\{ K , \{ S^k\}, \{ A(s) \}, T(s^k, a_k, W_k), R(s^k, a_k)\}
$
and our goal is to solve the following optimization problem:
\begin{align}
&\text{maximize}_{\pi} \; \sum_{k=1}^K \mathbb{E}^{\pi}\left[ R(S^k, \pi(S^k))  \right]& \notag \\ 
&\text{subject to }  \pi(S_k) \in A(S_k) \ \ \  \text{ for all } \ k  & \notag
\end{align}

In the next subsection, we propose a Dynamic Programming (DP) approach for solving the $K$-stage MDP defined above. 

\subsection{Dynamic Programming (DP) solution} \vspace{-.1in}

In the dynamic programming approach, the value function is defined as the optimal value function given a particular state $S^k$ at a particular stage $k$, and is determined recursively through Bellman's equation. If the value function can be computed efficiently, the optimal policy can also be computed from it. The optimal value function at the terminal stage $K-1$ (the stage in which the last patient is recruited) is given by: 
$$
V^{K-1}(s) = \max_{a \in A(s)} R(s, a)
$$
The dynamic programing principle tells us that the loss function at other indices $0 \leq k < K-1$ is given recursively by 
\begin{eqnarray*}
&Q^k(s, a) =  \mathbb{E}_k \left[ V^k(T(s, a, W))\right], \notag \\ 
&V^k(s) = \max_{a \in A(s)}  Q^k(s, a)
\end{eqnarray*}
where the expectation is taken with respect to the distribution of the cumulative treatment outcomes 
$W$. The dynamic programming principle tells us that any policy that satisfies the following is optimal: $A^{*}_k = \arg\max_{a \in A(S^k)} = Q^k(S^k, a)$. 

However, it is computationally intractable to solve the DP because the state space contains all possible distributions under consideration and so is very large. 
In what follows of the paper, we  propose an (approximate) solution under the restriction that the total number $M$ of  patients to be recruited at each step is fixed and determined by the designer. (As we show in the experiments, the choice of $m$ can have a large impact on the performance.) Our approach is greedy but computationally tractable. The proposed algorithm,  RCT-KG, computes for each action $a \in A$ the one-stage reward that would be obtained by taking $a$, and then selects the action with the maximum one-stage reward

\section{The RCT-KG Algorithm} \vspace{-.1in}

Because solving the MDP is intractable, we offer a tractable approximate solution.  We focus on the setting where the number of patients $M$ in each cohort is fixed: $M = T/K$. In this circumstance, what is to be decided for each cohort (each step) is the number of patients to recruit in each subgroup and the group-specific assignments of these patients (subject to the constraint that the total number of patients in each cohort is $M$).  The action set is therefore
$
A= \bigg\{ u: \sum_{x} \sum_{y} u(x,y) = M \bigg\}
$
and the size of this action set is $|A| =  {M + 2X -1 \choose 2X - 1}$. The Rand-KG algorithm computes an expected improvement in the terminal value function by taking an action $a \in A$. The value function at the terminal stage can be decomposed into improvements in the value function, that is, 
\begin{align*}
V^K(S^K) &= \left[ V^K(S^K)  - V^K(S^{K-1})  \right] \\
&\;\;\;\;\; + \ldots + \left[ V^K(S^{k+1})  - V^K(S^k)  \right] + V^K(S^k)
\end{align*}
The Knowledge Gradient (KG) policy selects the action that  makes the largest improvement in the value function at each instance, that is, 
\begin{align*}
A^{KG}_k(s) &=  \arg\max_{a \in A(s)} \; \mathbb{E}_k\left[ V^K(P(s, a, W)) - V^K(s)\right] &  \\ 
&=  \arg\max_{a \in A(s)}  \; \int_{w} V^K(P(s, a, w)) \mathbb{P}(w | S^k) \,dw &
\end{align*}

\begin{algorithm}
\caption{Optimistic Action Computation}
\label{alg1}
\begin{algorithmic}
\STATE \textbf{Input :} Current state vector: $s$
\STATE Set optimal action $u^{*} = 0$
\FOR{$m = 1, \ldots, M$}
\FOR{$(x,y) \in X \times Y$}
\STATE Set $s_1 = P(s, u^{*}, u^{*} \odot G(z_{\max}))$
\STATE Set $s_2 = P(s, u^{*}, u^{*} \odot G(z_{\min}))$
\STATE Set $\tilde{u} = u^{*} + 1_{(x,y)}$
\STATE Compute $v_1 = V^K(P(s, \tilde{u}, \tilde{u}  \odot G(z_{\max}))) - V^K(s_1)$
\STATE Compute $v_2 = V^K(P(s, \tilde{u}, \tilde{u}  \odot G(z_{\min}))) - V^K(s_2)$
\STATE Compute $q(x,y) = \max(v_1, v_2)$.
\ENDFOR
\STATE Compute $(x^{*}, y^{*}) = \arg\max_{x,y} q(x,y)$.
\STATE Update $u^{*} = u^{*} + 1_{(x^{*}, y^{*})}$.
\ENDFOR
\STATE Return $A^{RCT-KG}(s) = u^{*}$.
\end{algorithmic}
\end{algorithm}

However, computing the KG policy requires computing the posterior predictive distribution and posterior expectation  for each action in $A$. This is a computationally intractable procedure because the size of the action space is on the order of $\mathcal{O}\left( M^{2X -1} \right)$. Hence we propose the RCT-KG algorithm which computes {\em optimistic} improvements in the value functions. Algorithm 1 shows a tractable way of computing these optimistic improvements. At each iteration $m$, the procedure computes the maximum improvement in the value function that can be obtained from an additional sample from the pair $(x,y)$ and increments that index by $1$. The complexity of computing $A^{RCT-KG}(s)$ is only $\mathcal{O}(M X)$.
\begin{algorithm}
\caption{The RCT-KG Algorithm}
\label{Alg:RCT-KG}
\begin{algorithmic}
\STATE \textbf{Input :} $K$, $S^0$
\FOR{$k = 1, \ldots, K$}
\STATE Compute $U_k^{*} = A^{RCT-KG}(S^k)$ using Algorithm 1. 
\STATE Recruit the patients based on $U_k^{*}$, observe cumulative treatment outcome $W_k^{*}$. 
\STATE Update $S^{k+1} = S^k + (W_k^{*}, U_k^{*})$. 
\ENDFOR
\STATE Compute $P_x(S^K)$ for all $x$.
\STATE Compute ${H}^+_{\rm est} = \{ x \in \mathcal{X}: P_x(S^K) \geq 1- \lambda \}$.
\STATE \textbf{Output :} ${H}^+_{\rm est}$. 
\end{algorithmic}
\end{algorithm}

At each decision step $k$, the algorithm computes the best action using the procedure in Algorithm \ref{alg1} and recruits and assigns  patients accordingly. At the end of the current decision step, the state vector is updated based on the observed treatment outcomes.  When the patient budget is exhausted, our algorithm outputs (as the estimated positive set) the  set of subgroups with clinically relevant improvements:
$$
{H}^+_{\rm est} = \{ x \in X: P_x(s^K) \geq 1- \lambda \}
$$
The pseudo-code for RCT-KG is given in Algorithm \ref{Alg:RCT-KG}. 

\section{Experiments} \vspace{-.1in}

In all of our experiments, we use assume that the outcome is either success or failure of the treatment for that patient, and assume the outcomes follow a Bernoulli distribution; we continue to assume that the outcomes for each cohort of patients is observable before the succeeding cohort must be recruited and allocated.  (The clinical study in \cite{hebert1999multicenter} provides a real-life example.)  We assume there are identifiable subgroups (e.g. distinguished as to male/female and/or young/old).  In each experimental setting we generate 1000 independent experiments and report the average over these 1000 experiments.  We compare the results of our algorithm with those of  Uniform Allocation (UA), aka repeated fair coin tossing, that uniformly randomly recruits the patients from subgroups and uniformly assigns the patients to treatment groups, and (where appropriate) with Thompson Sampling (TS)~\cite{agrawal2012analysis} that draws the parameters of the outcomes for each action and then selects the action with the best sample outcomes and an adaptive randomization, and a variant of DexFEM~\cite{warner2015low}, that shifts the treatment allocation ratio towards treatments with higher posterior variance. In the comparisons with TS and DexFEM, we assume that the recruitment of patients from subgroups is uniform and that it is only the allocation of patients to treatment/control that depends on previously observed outcomes.

For the first experiment we consider a setting with two subgroups; for the remaining experiments, we consider a setting with four subgroups.
\vspace{-.1in}
\subsection{Error Rates with Two Subgroups} \vspace{-.1in}

We begin with a setting with 2 subgroups 0, 1.  We assume the true parameters for the subgroups are $\theta_{x,0} = 0.5$ for  $x \in \{0,1\}$,  $\theta_{1,1} = 0.7$ and we vary $\theta_{0,1}$ from $0.51$ to $0.70$. (Note that identifying the best treatment is more challenging for subgroup 0 than for subgroup 1.) We recruit 100 patients in each of 10 cohorts -- 1000 patients in total.  Figure \ref{fig:effic_gap} compares the performance of our RCT-KG with UA in terms of error rates; as can be seen, our algorithm outperforms UA throughout the range of the parameter $\theta_{0,1}$, and the improvement in performance is greatest in the middle range of this parameter, when the difference  between treatment and control among the subgroups is greatest. This improvement is achieved because our algorithm recruits and samples more frequently from subgroup 0, which represents the more challenging learning problem

\subsection{Error Rates and Confidence Levels with Four Subgroups} \vspace{-.1in}

We now turn to a setting in which there are 4 subgroups 0,1,2,3.  We take the parameters to be $\theta_{x,0} = 0.5$ for  all $x$ and $\theta_{x,1}  = 0.3, 0.45, 0.55, 0.7$ for $x = 0,1,2,3$.  Note that there are two subgroups in which the treatment action is ineffective and two subgroups in which it is effective, and that identification is easier in  subgroups $0,3$ than in subgroups $1,2$.   We examine a number of aspects.

\subsubsection{Confidence Levels Across Cohorts} \vspace{-.1in}

In the first experiment in this setting, we recruit 100 patients in each of 10 cohorts, and compare the confidence levels achieved for each cohort and for each subgroup; the results for 4 horizons are shown in Table \ref{table:comp_conf}.  As can be seen, the RCT-KG, UA and DexFEM algorithms achieve very similar confidence levels (probability of correctly identifying the actual label) for subgroups 0,3 at each of these time horizons, but RCT-KG algorithm achieves significantly better confidence levels for subgroups 1,2 -- the subgroups for which identification is more difficult.  RCT-KG achieves superior confidence levels because it recruits more patients to subgroups 1,2 and allocates patients in each subgroup more informatively.  To illustrate, we refer to Table \ref{table:num_patient_subgroup}, which shows the total number of patients recruited to each of the subgroups and the allocation of patients to control and treatment within the subgroups. (Remember that we are reporting averages over 1000 experiments.)  

\begin{figure*}
    \centering
    \subfigure[Type-I error]{\includegraphics[width=0.32\textwidth]{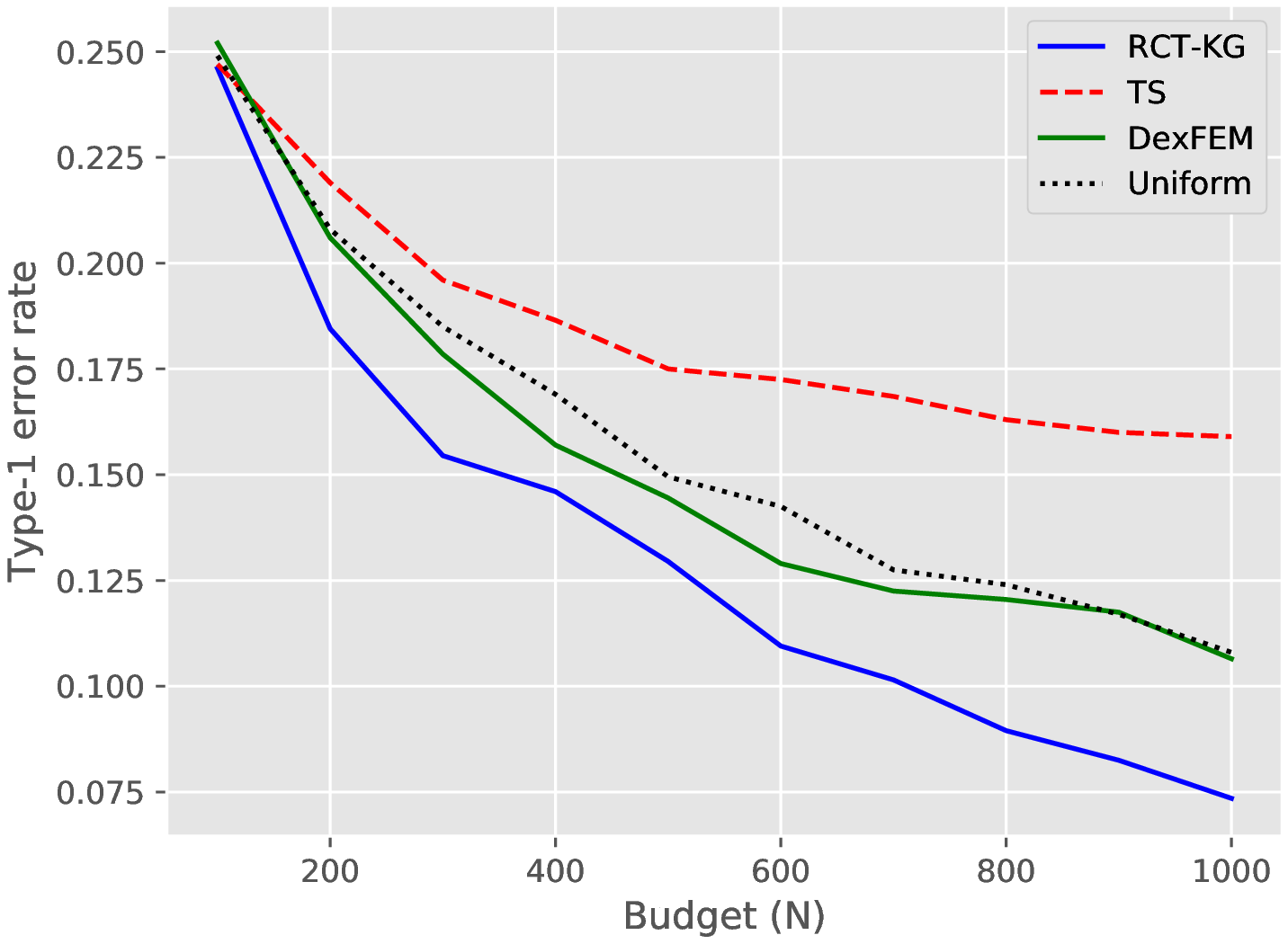}}
    \subfigure[Type-II error]{\includegraphics[width=0.32\textwidth]{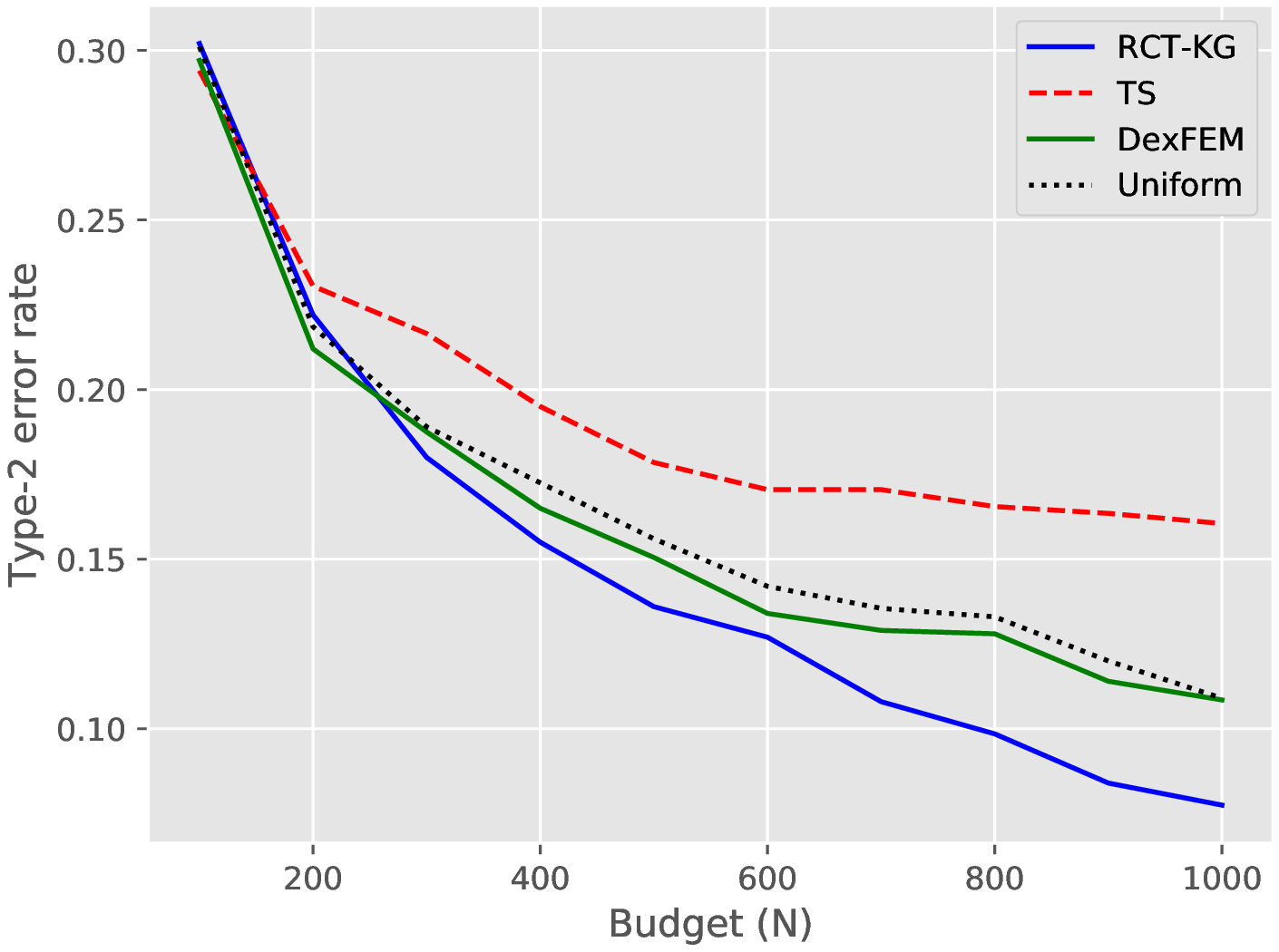}}   
    \subfigure[Total error]{\includegraphics[width=0.32\textwidth]{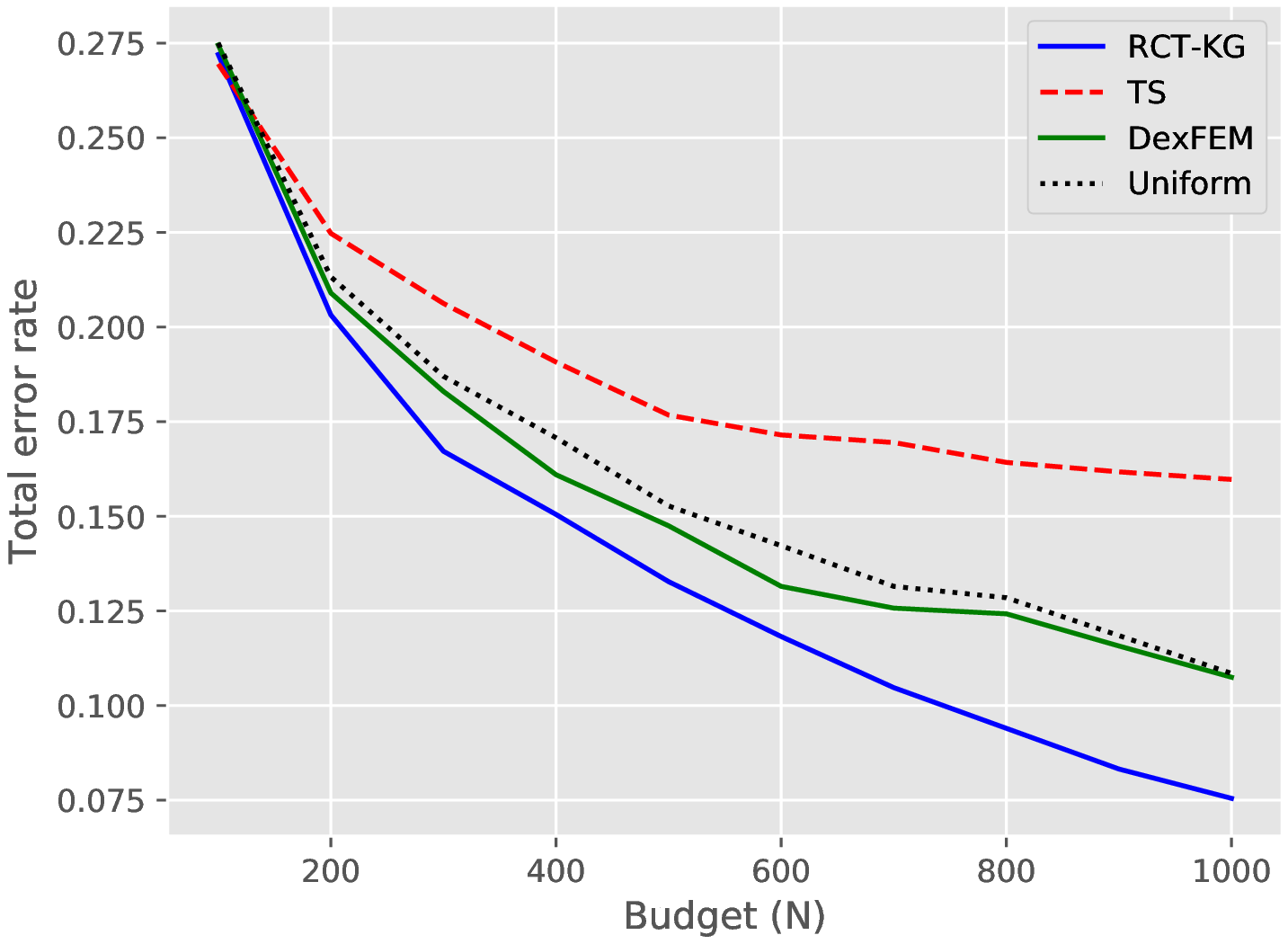}}        
    \caption{Error Comparisons with Benchmarks}
            \label{fig:effic_budget}       

\end{figure*}

\begin{figure}[h!]
\centering   
    \includegraphics[width=0.99\linewidth]{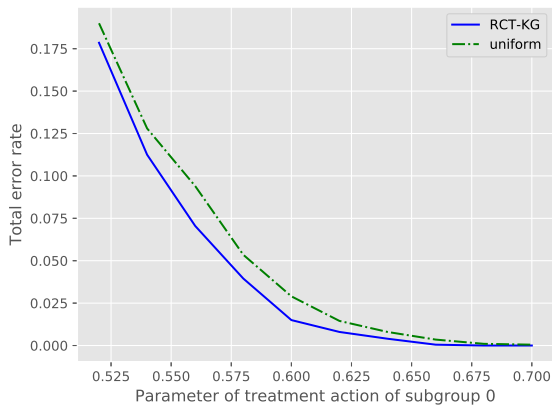}     
    \caption{Total error rates for different parameter}
        \label{fig:effic_gap}
\end{figure}

\begin{table}[h!]
\centering
\begin{tabular}{ | c | c | c | c | c |   }
    \hline
    Algorithm/ SG & $0$ & $1$ & $2$ & $3$  \\ \hline
    RCT-KG & $98.79$& $82.92$ & $83.56$ & $98.78$ \\ \hline
    DexFEM &$98.94$& $79.28$ & $79.25$& $99.00$ \\ \hline
    UA & $98.92$& $78.93$ & $78.96$& $98.94$ \\ \hline
\end{tabular} 
\caption{Comparison of confidence levels on subgroups}
\label{table:comp_conf}
\end{table}

\begin{table}[h!]
\centering
\begin{tabular}{ | c | c | c | c | c |}
    \hline
    SG & $0$ & $1$ & $2$ & $3$  \\\hline
    RCT-KG & $81, 60$ & $190, 181$ & $186, 177$ & $81, 54$ \\ \hline
    DexFEM & $137, 118$ & $128, 127$ & $128, 127$ & $137, 118$ \\ \hline
    UA  & $128, 127$ & $127, 128$ & $128, 127$ & $127, 128$ \\ \hline
    \end{tabular}
    \vspace{0.2in}
\caption{Recruitment and allocation in each subgroup \\ $\#$ allocated to control, 
$\#$ allocated to treatment}
\label{table:num_patient_subgroup}
\end{table}

\subsubsection{Achieving a Prescribed Confidence Level} \vspace{-.1in}

In this  experiment, we recruited 100 patients in each cohort and continued recruiting patients until a prescribed (average) confidence level $\beta = 0.90, 0.95$ was achieved (i.e., until $\frac{1}{X} \sum_{x \in \mathcal{X}} g(P_x(S^k)) < 1- \beta$.)  Table \ref{table:result_trial_length} shows the number of cohorts necessary for each algorithm to achieve the prescribed confidence level; as can be seen, RCT-KG achieves the same confidence level as UA and DexFEM using fewer cohorts of patients, which means that a RCT could be carried out with fewer patients and completed in less time.

\begin{table}[h!]
\centering
\begin{tabular}{ | c | c | c |}
    \hline
    Algorithm & $\beta = 0.95$ & $\beta = 0.90$ \\ \hline
    RCT-KG & $12.6$ & $7.2$ \\ \hline
    DexFEM & $22.5$ & $10.2$ \\ \hline
    UA & $22.9$ & $10.7$ \\ \hline
    \end{tabular}
\caption{Comparison of length for a confidence level}
\label{table:result_trial_length}
\end{table}

\subsubsection{Error Rates and Patient Budgets} \vspace{-.1in}

In this experiment we recruited 100 patients in each cohort and computed the type-I, type-II and overall error rates for various total patient budgets.  As seen from Figure \ref{fig:effic_budget}, RCT-KG significantly outperformed the UA, DexFEM and TS algorithms for all budgets.  (TS did especially poorly when the patient budget is large because TS aims to maximize the patient benefit, not the learning performance, and so allocated more patients to the treatment that has been found to be better at each stage, which slows learning.)

\subsubsection{Cohort Size} \vspace{-.1in}

In this experiment, we compared the performance in terms of total error of RCT-KG and UA when  $m$ patients were recruited in each cohort with a total  budget of $500$ patients. As seen in Table \ref{table:metrics_vs_m}, the error rate of UA is independent of $m$ because all patients are recruited and allocated in the same way in every cohort, but the error rate of RCT-KG is significantly smaller when $m$ is smaller because smaller cohorts provide more opportunities to learn.

\begin{table}[h!]
\centering
\begin{tabular}{ | c | c | c | c | c|  }
    \hline
    $m$ & $25$ & $50$ & $100$ & $250$  \\\hline
    RCT-KG 	 & $0.1245$ & $0.1281$ & $0.1292$ & $0.1411$ \\ \hline
    UA &  $0.1484$ & $0.1484$ & $0.1484$ & $0.1484$  \\ \hline
    \end{tabular} 
\caption{Total Errors for Different Cohort Sizes}
\label{table:metrics_vs_m}
\end{table}

\subsection{Type 1 and Type 2 errors}
Figure \ref{fig:tradeoff}  illustrates the trade-off between  type-I and type-II errors for the RCT-KG algorithm (100 patients in each of 10 cohorts.) Of course, as $\lambda$ increases, the type-I error decreases and the type-II error increases. 

\begin{figure}[h!]
\centering   
    \includegraphics[width=0.49\textwidth]{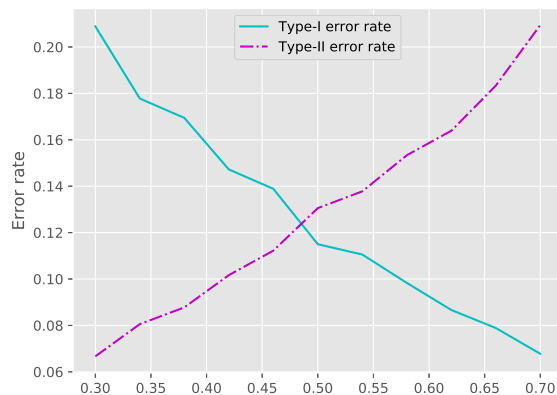} \label{fig:tradeoff}
    \caption{Tradeoffs between Type-I and Type-II errors}
\end{figure}

\subsection{Informative Prior}
In many circumstances, the clinicians controlling the trial may have  priors about the parameters of the outcome distribution; of course these priors may be informative but not entirely correct.  To illustrate the effect of having such informative priors, we conducted an experiment in which informative priors were generated by sampling $50$ patients from each subgroup and treatment group, but noting count these patients as part of the total patient budget, which was either 500 or 1000.  As seen from  Table \ref{table:informative}, having an informative prior increases the amount by which RCT-KG improves over UA.  Having an informative prior is useful because it allows RCT-KG to make more informed decisions in the earlier stages and in particular to focus more on subgroups for which the difference between the true effectiveness of the treatment and control is smaller.

\begin{table}[h!]
\centering
\begin{tabular}{ | c | c | c |  }
    \hline
    Budget & $500$ & $1000$  \\\hline
    Non-informative 	 & $0.1151$ & $0.2813$  \\ \hline
    Informative for sg 0,3  & $0.2000$ & $0.3703$    \\ \hline
    Informative for sg 1,2  & $0.1323$ & $0.2878$    \\ \hline
    \end{tabular} 
\caption{Improvement score for different budgets}
\label{table:informative}
\end{table}

\section{Conclusion and Future Work}\vspace{-.1in}
 This paper makes three main contributions.  (1) We formalize the problem of recruiting and allocating patients in a RCT as a finite stage MDP. (2) We provide a greedy computationally tractable algorithm RCT-KG that provides an approximately optimal solution to this problem. (3) We illustrate the effectiveness of our algorithm in a collection of experiments using synthetic datasets for which outcomes are drawn from a Bernoulli distribution with unknown probabilities.

 The most important assumptions  of this paper are that patients can be recruited from subgroups that are identified in advance and that the final outcomes for patients in each cohort can be observed before patients in the next cohort are recruited from the various subgroups and allocated to treatment/control.  In future work we will address both of these assumptions.  In particular, we will address the settings in which subgroups are not identified in advance but must be learned during the course of the trial, and the setting in which only partial information about the outcomes of earlier cohorts is known before patients in the next cohort are recruited and allocated.
 
 We have noted that this paper makes specific assumptions in order that our approach can be followed.  We should also add that actually constructing trials according to the method suggested here will require convincing both those who conduct the trials (e.g. pharmaceutical companies) and those who assess the results of the trials (e.g. the regulatory agencies) that the substantial improvements that are possible using our method justify the changes to the way trials are presently conducted.

\bibliographystyle{abbrvnat}

\end{document}